\documentclass[runningheads]{llncs}
\usepackage[T1]{fontenc}
\usepackage{graphicx}

\usepackage{cite}
\usepackage[linesnumbered,ruled,vlined]{algorithm2e}
\usepackage{mathrsfs}
\usepackage{amsmath,amssymb,amsfonts}
\usepackage{algorithmic}
\usepackage{multirow} 
\usepackage{booktabs}
\usepackage{textcomp}

\usepackage{sidecap}
\usepackage{graphicx} \usepackage[font=footnotesize]{caption}
\usepackage{array}  
\usepackage{makecell}  

\usepackage{cases}

\usepackage{graphicx}
\usepackage{subcaption}
\usepackage{setspace}
\usepackage[table,xcdraw]{xcolor} 
\usepackage{tabularx}
\usepackage{booktabs} 
\begin{document}

\title{A Service-Oriented Adaptive Hierarchical Incentive Mechanism for Federated Learning\\}
\titlerunning{Abbreviated paper title}
%
 \author{Jiaxing Cao 
 \and
 Yuzhou Gao \and
 Jiwei Huang}
 
 \authorrunning{F. Author et al.}

 \institute{China University of Petroleum, Beijing }

\maketitle              
\begin{abstract}
Recently, federated learning (FL) has emerged as a novel framework for distributed model training. In FL, the task publisher (TP) releases tasks, and local model owners (LMOs) use their local data to train models. Sometimes, FL suffers from the lack of training data, and thus workers are recruited for gathering data. To this end, this paper proposes an adaptive incentive mechanism from a service-oriented perspective, with the objective of maximizing the utilities of TP, LMOs and workers. Specifically, a Stackelberg game is theoretically established between the LMOs and TP, positioning TP as the leader and the LMOs as followers. An analytical Nash equilibrium solution is derived to maximize their utilities. The interaction between LMOs and workers is formulated by a multi-agent Markov decision process (MAMDP), with the optimal strategy identified via deep reinforcement learning (DRL). Additionally, an Adaptively Searching the Optimal Strategy Algorithm (ASOSA) is designed to stabilize the strategies of each participant and solve the coupling problems. Extensive numerical experiments are conducted to validate the efficacy of the proposed method.

\end{abstract}
\section{Introduction}
In recent years, with the rapid development of mobile computing, hundreds of thousands of data points are created on mobile terminals almost every moment\cite{gaff2014privacy}. Within the framework of traditional cloud computing, how to securely and efficiently upload vast amounts of data to central servers has become a significant challenge\cite{wu2014service}\cite{liu2024blockchain}.

Federated learning (FL) is an emerging computational paradigm\cite{li2020federated}. FL can be used in machine learning (ML)\cite{yang2019federated}, where different devices use distributed data to generate local models. After the local training process, local models are uploaded to the task publisher (TP) by the local model owner (LMO). A global model is aggregated through algorithms like FedAvg\cite{mcmahan2017communication}, FedProx\cite{li2020federated} and SCAFFOLD\cite{karimireddy2020scaffold}. This distributed training approach avoids privacy issues caused by uploading sensitive data to servers, while also reducing the communication costs of uploading large amounts of data.

In 2018, Google successfully developed its predictive keyboard input system\cite{hard2018federated}, which opened up a wide range of application scenarios for federated learning. Researchers also use federated learning for COVID-19 diagnosis\cite{zhang2021dynamic} and the development of autonomous driving technologies\cite{li2021privacy}. These examples demonstrate that federated learning possesses extensive application prospects.

Specifically, federated learning primarily involves two processes. First, the Task Publisher (TP) releases a task, then Local Model Owners (LMOs) conduct training locally and upload their local model parameters. This entire process repeats continuously until the model's loss function is minimized, indicating an optimal global model has been obtained. Throughout this process, a reasonable incentive mechanism needs to be designed to encourage maximum participation from LMOs in training tasks. An effective incentive mechanism should provide timely feedback based on LMO performance and continuously promote high-quality model training by LMOs\cite{liu2020systematic}. Ultimately, this allows both the TP and LMOs to achieve optimal utility. However, existing studies mainly focus on finding analytical solutions. In scenarios with limited data, how to incentivize heterogeneous LMOs and dynamically adjust the strategy has not been thoroughly investigated.

In this paper, we design a hierarchical game framework. In the upper layer, Task Publishers (TP) and Local Model Owners (LMO) engage in a Stackelberg game. In the lower layer, LMOs recruit workers to complete data collection tasks based on a Multi-Agent Markov Decision Process (MAMDP). Two main challenges exist: First, how to couple the upper layer (Stackelberg game) and the lower layer (MAMDP). Second, how to dynamically adjust strategies for the different capabilities of each LMO.

The main contributions of this paper are summarized as follows:

\begin{enumerate}
\item We propose a three-layer game framework that includes Task Publishers (TP), Local Model Owners (LMO), and workers. We use Stackelberg game theory to analyze the relationship between TP and LMOs. A Markov Decision Process is employed to describe the interaction between workers and LMOs.
\item In the upper layer, we identify the Nash equilibrium relationship among LMOs, which leads us to determine the Stackelberg equilibrium between the TP and LMOs. In the lower layer, using Deep Reinforcement Learning (DRL), we develop an optimal strategy that satisfies worker benefits, completing data collection tasks with maximum efficiency.
\item We propose the algorithm: Adaptively Searching the Optimal Strategy Algorithm (ASOSA), which addresses the coupling between the upper layer (Stackelberg game) and the lower layer (MAMDP) by dynamically adjusting the coupling factor. Under performing LMOs are promptly eliminated during the process, leading to a stable final strategy.
\item We first validate the performance of ASOSA using real-world datasets MNIST and Fashion-MNIST. We then compare it with other strategies. The results show that ASOSA rapidly responds to the performance of different LMOs and ultimately achieves a stable strategy. In comparative experiments, ASOSA significantly outperforms other baseline strategies.
\end{enumerate}
The rest of the paper is organized as follows: Section 2 reviews related works. Section 3 introduces the system model and problem formulation. Section 4 shows the equilibrium analysis between TP and LMOs. Section 5 addresses the MDP and system iteration algorithm ASOSA. Section 6 evaluates the performance of ASOSA and shows the comparison experiments. At the end of the paper, Section 7 sums up the whole paper and discusses the possibilities for future research.


\section{Related Works}
\subsection{Federated Learning}
Federated Learning (FL) is a distributed machine learning paradigm that allows multiple participants to collaboratively train shared models while protecting data privacy. The core idea is that data is kept local and only model parameters are transmitted, thus effectively mitigating data silos and privacy protection issues, while improving training efficiency and reducing the computational and storage burden on a single server \cite{rb1}. Compared with the traditional centralized training mode, federated learning shows significant advantages in scenarios with strong data sensitivity and high demand for distributed computing. It is widely used in many fields such as medical and healthcare \cite{rb2}, intelligent industry \cite{rb3}, and Internet of Vehicles \cite{rb4}, showing great application value and development potential.

In recent years, in order to further optimize the federated learning framework, many researchers have proposed optimization schemes. Wang et al.\cite{rb5} propose AIGC-augmented Federated Preference Learning (FPL) to train specific preference classes, enhancing data quality through pre-training and fine-tuning. Zhang et al.\cite{rb6} propose Double-Blind Collaborative Learning (DBCL), using random matrix sketching on model parameters to prevent gradient-based privacy inferences. 

In addition, the efficient operation of federated learning relies on the motivation of each participant, however, some nodes may lack sufficient motivation to participate due to factors such as computational cost and data quality differences. Therefore, how to design an effective incentive mechanism to encourage all parties to contribute fairly and enhance the overall training effect becomes an important problem to be solved in the field of federated learning.

\subsection{Incentive Mechanism Design}
For the design of incentive mechanisms for federated learning scenarios, one part of the research focuses on the design of incentive mechanisms for the architecture of TP-LMO \cite{rb10},\cite{rb11}, while the other part further considers the Worker layer and constructs effective incentive mechanisms applicable to the TP-LMO-Worker architecture \cite{rb15}. Whether it is a two-layer or three-layer architecture, the existing methods can be roughly categorized as follows: game-based \cite{rb10},\cite{rb11}, auction-based \cite{rb12},\cite{rb13},\cite{li2024dynamic}, contract-based \cite{rb14},\cite{rb15}, and reinforcement learning-based methods \cite{rb12},\cite{rb13}.

Zhang et al.\cite{rb10} propose a Stackelberg game-based incentive model that rewards clients based on model quality, solving two optimization problems with a linear complexity algorithm to achieve optimal solutions and reduce computational cost. Cho et al.\cite{rb11} model FL with multiple requesters and unconstrained workers as a Stackelberg game, proposing a method to find equilibrium, analyze interplays, and adapt to constraints while confirming requesters' first-mover advantage and potential mutual benefits. \cite{rb12} proposes a two-layer Reinforcement Auction Mechanism (RAM): the upper layer uses Hybrid Action Reinforcement Learning for user selection and payments, while the lower layer optimizes resource allocation for utility maximization. \cite{rb13} proposes a long-term adaptive VCG auction mechanism with a multi-branch deep reinforcement learning (DRL) algorithm to address FL's incentive challenges. \cite{rb14} proposes an FL-based architecture for privacy-preserving collaborative learning among DaaS providers in IoT applications. To address incentive mismatches and information asymmetry, it applies contract theory with a self-disclosure mechanism for truthful UAV capability reporting while optimizing the model owner's profits. Ding et al.\cite{rb15} derive optimal contracts and pricing mechanisms under three interaction structures and propose multi-period pricing to simplify the dynamic pricing.

Unlike previous research, we separately discuss the theoretically optimal strategy generated through TP and LMOs, as well as the actual values produced by LMOs and workers. We address the coupling problem between these two through our algorithm design. We focus particularly on how the unit data acquisition cost affects strategies at both the upper and lower layers. We conduct the following work to develop an adaptive algorithm based on different LMO capabilities.

\section{System Model and Problem Formulation}
In this section, we first introduce the framework of FL with a two-layer game. Next we formulate an optimization problem to maximize the utility functions of TP, LMOs and workers. For convenience, the key notations used are summarized in Table I.

\begin{figure}[htbp]    
  \centering            
{
      \label{fig:subfig1}\includegraphics[width=0.9\textwidth]{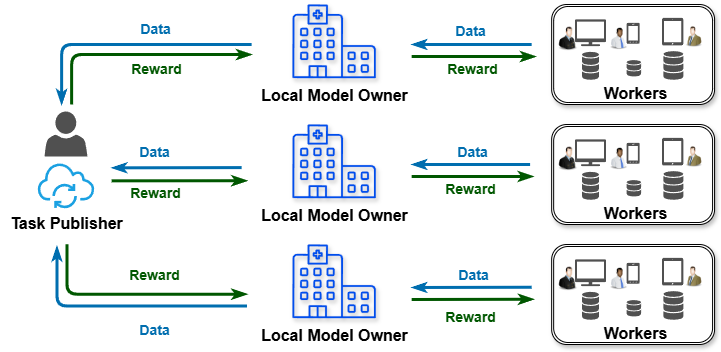}
  }
  \caption{The system framework.}    
  \label{fig:subfig_1}            
\end{figure}

\begin{table}[tb]
    \caption{Main Notations}
    \renewcommand{\arraystretch}{1.3}
    \begin{center}
    \rowcolors{2}{gray!15}{white}
    \begin{tabularx}{\linewidth}{>{\raggedright\arraybackslash}p{2.5cm} >{\raggedright\arraybackslash}X}
        \rowcolor{gray!35}
        \textbf{Notation} & \textbf{Explanation} \\
        $\tau$ & Total budget of the task publisher (TP) \\
        $M$ & Number of local model owners (LMOs) \\
        $\mathbb{Z}$ & Total data collected from all LMOs \\
        $\zeta_m$ & Private dataset of LMO $m$ \\
        $B_m$ & Budget of LMO $m$ \\
        $B_m^*$ & Optimal budget of LMO $m$ \\
        $\tilde{B}_m$ & Remaining budget of LMO $m$ \\
        $P_m$ & Unit data purchase cost of LMO $m$ \\
        $W_m$ & Number of workers under LMO $m$ \\
        $d_m^i$ & Data contribution of worker $i$ under LMO $m$ \\
        $f(d_m^i)$ & Fatigue function of worker $i$ \\
        $\lambda$ & Conversion coefficient parameter \\
    \end{tabularx}
    \label{tab:notations}
    \end{center}
\end{table}

\subsection{FL Framework with Two-layered Game}
As shown in Fig.1, this framework involves three types of characters: TP, LMOs, and workers. In the upper layer, the TP announces model training tasks and budgets, which attract LMOs to participate in the task to earn the reward. In the lower layer, to complete local model training, LMOs usually need data provided by workers. The workers collect data and sell it to the LMOs.

In the upper layer, LMOs are independent of each other in our framework. During the interaction between LMOs and TP, LMOs calculate the most suitable data required to train their local model and gather these data on their own. That means data does not need to be pooled or shared on the central server. This design allows LMOs to train models independently using local data, improving data privacy and the specificity of the models.

In the lower layer, interactions take place between the LMO and the workers. The primary responsibility of the LMO in this layer is to recruit workers for data collection. The data collected by the workers forms a dataset that is exclusively accessible to the LMO who hired them. In our framework design, data is typically scarce and valuable. Therefore, we give the decision-making authority to the workers. From the workers' perspective, they need to strategically decide whether to participate in data collection to maximize their own benefits.

Overall, this architectural design ensures data privacy, while avoiding latency during data transmission and minimizing additional transmission costs. It enables the training tasks to be completed efficiently and safely.

\subsection{Problem Formulation}
We denote $\tau$ to represent the total budget of the task publisher in the training task. Denote ${M}$ = {1,2,3..., $\mathcal{M}$} as the set of LMOs who want to participate in the training task. Each $LMO_m$ $(m\in {M})$ has the private dataset $\zeta_m$ collected by the workers. The total data of all LMOs is $\mathbb{Z} = \sum_{m\in {M}}^{}\zeta_m $.

    \subsubsection{Utility of TP}
    $\tau$ represents the total budget that is published by the task publisher. Similar to the previous studies\cite{zhao2023multi}, the utility function of TP is dependent on $\tau$ and the total data volume $\mathbb{Z}$, which is expressed as follows:
    \begin{equation}
    \mathcal{U}(\tau)=\lambda g(\mathbb{Z})-\tau\label{1}
    \end{equation}
    where, $\lambda$ is a system parameter, which represents the conversion coefficient between model accuracy and revenue.$g(\mathbb{Z})$ is a function of model accuracy with respect to the total number of data samples $\mathbb{Z}$. According to related works, as the data volume increases, the model's benefits exhibit diminishing marginal returns. Consistent with related studies\cite{huang_hierarchical_2024}, the expression for $g(\mathbb{Z})$ is given as follows:
    \begin{equation}
    g(\mathbb{Z}) = \alpha ln(1+\beta \mathbb{Z})\label{2}
    \end{equation}
    where $\alpha$ and $\beta$ are system parameters, their value ranges will be given respectively in the following sections.
    
    \subsubsection{Utility of LMO}
    In our structure, LMOs train the local models for profit with their local datasets. We assume that the TP allocates a budget to each $LMO_m$ based on their respective data contributions. The budget of $LMO_m$ is:
    \begin{equation}
    B_m =\frac{\zeta_m}{\mathbb{Z}}\tau\label{3}
    \end{equation}
    For $LMO_m$, the primary cost is the unit data purchase cost $P_m$ paid to the workers for data collection. Additional costs include communication and computation expenses. To simplify the notation, we define $E_m$ as the total cost excluding the data purchase cost, which is a specific coefficient associated with each $LMO_m$. Thus, the utility function of $LMO_m$ is expressed as follows:
    \begin{equation}
    \mathcal{U}_m(\zeta_m,\tau,\mathbb{Z}) = B_m -P_m \zeta_m - E_m\label{4}
    \end{equation}
    \subsubsection{Utility of Worker}  
    For workers, we define $d_m^i$ as the amount of data collected by $worker_i$ for $LMO_m$. The number of workers of $LMO_m$ is $W_m ={1,2...,w_m}$. The fatigue function $f(d_m^i)$ reflects the relationship between the worker's fatigue level and their data contribution. Considering the reality of the situation, a worker's fatigue level does not increase linearly with the amount of collected data. Instead, fatigue accumulates gradually at first, then rises rapidly and finally slows down. Based on this observation, we adopt a specially designed sigmoid function:
    \begin{equation}
    f(d_m^i)=\frac{\epsilon}{1+e^{-(\gamma(d_m^i - \delta)-0.5)}}\label{5}
    \end{equation}
    where, $\gamma$, $\epsilon$ and $\delta$ are parameters to control the rate and range at which fatigue levels increase. In order to increase the willingness of workers to contribute data, we provide a base participation reward:
    \begin{equation}
        R_{base}(\tilde{B}_m,B_m)=\theta(1-\frac{\tilde{B}_m}{B_m})
    \end{equation}
    where $\theta$ is a parameter and $\tilde{B}_m$ represents the residual budget of $LMO_m$. This reward gradually increases as the game progresses. Therefore, the utility function of the worker will be:
    \begin{equation}
    \mathcal{U}_m^i(d_m^i,B_m,\tilde{B}_m) = \frac{d_m^i}{\sum_{i\in w_m}d_m^i}{B}_m+R_{base}-f(d_m^i)\label{6}
    \end{equation}
    
    \subsubsection{Optimization Problems}
    In our framework, the TP must determine the total budget to attract LMOs to participate in the training task, thereby ensuring the acquisition of a high-quality model. LMOs need to determine the optimal data contribution level to minimize costs while completing local training effectively. For workers, the challenge lies in learning how to strategically participate in LMO data collection tasks to maximize their own utility. Overall, the core problem of the entire framework can be formulated as follows:
    \begin{equation}
    TP: Maximize \mathcal\ {U}(\tau)
    \end{equation}
    \begin{equation}
    LMO:Maximize\ \mathcal{U}_m(\zeta_m,\tau,\mathbb{Z})
    \end{equation}
    \begin{equation}
    Worker:Maximize\ \mathcal{U}_m^i(d_m^i,B_m,\tilde{B}_m)
    \end{equation}
    
\section{Equilibrium Analysis}

    \subsection{Strategy of LMO}
    For each $LMO_m$, the key decision is determining the optimal data contribution $\zeta_m$. Each LMO aims to maximize its own utility. However, since LMOs tend to compete with each other and they are unwilling to share data (e.g. hospitals and banks), their interactions can be viewed as a non-cooperative game. The optimal strategies of the LMOs will naturally lead to a Nash equilibrium. Based on previous studies\cite{huang_hierarchical_2024}, we derive Definition~\ref{definition} as follows:
    
    \begin{definition}\label{definition}
        During interaction between LMOs, there exists a Nash Equilibrium $\zeta^* = (\zeta_1^*, \zeta_2^*, \dots, \zeta_\mathcal{M}^*)$ of the Stackelberg game, 
        \[ \mathcal{U}_m(\zeta_m^*,\zeta_{-m}^*)\geq \mathcal{U}_m(\zeta_m,\zeta_{-m}^*).\]
    \end{definition}
    
    $\mathbb{Z}$ is the total data contribution of system, so we define $\mathbb{Z}$ as:
    \begin{equation}
        \mathbb{Z} = \sum_{n=1}^{M}\zeta_n
    \end{equation}

    Next, Theorem 1 will determine the optimal value of $\mathbb{Z}^*$ through Nash equilibrium conditions.
    
    \begin{theorem}\label{theorem2}
        The optimal strategy of total data contribution $\mathbb{Z}$ is: 
        \begin{equation} \label{10}
            \mathbb{Z}^*  = \frac{(M-1)\tau }{\sum\limits_{n\in M}P_n}
        \end{equation}
    \end{theorem}

    \begin{proof}
        First, we can obtain the first and second order derivative of the utility of $LMO_m$ with respect to $\zeta_m$:
        \begin{equation}
             \frac{\partial \mathcal{U}_m}{\partial \zeta_m}=\tau\left ( \frac{\mathcal{Z}-\zeta_m}{\mathcal{Z}^2} \right )-P_m
        \end{equation}
        \begin{equation}
            \frac{\partial^2 \mathcal{U}_m}{\partial \zeta^2}=-\frac{2\tau(\mathbb{Z}-\zeta_m)}{\mathbb{Z}^3}
        \end{equation}
        
        Evidently, Eq.(14) is less than zero. Therefore, a Nash equilibrium must exist for this Stackelberg game. The optimal strategy $\zeta_m$ can be derived through the following process. 
        \begin{equation}
            \frac{\partial \mathcal{U}_m}{\partial \zeta_m}=\tau\left ( \frac{\mathcal{Z}-\zeta_m}{\mathcal{Z}^2} \right )-P_m=0
        \end{equation}
        So that we can obtain the optimal strategy of $LMO_m$:
        \begin{equation}
            \zeta_m^*=\mathbb{Z}-\frac{P_m \mathbb{Z}^2}{\tau}
        \end{equation}

        Base on the Definition 1, the total optimal data contribution of all LMOs can be calculate:
        \begin{numcases}{}
        \zeta_1^*= \mathbb{Z}-\frac{P_1 \mathbb{Z}^2}{\tau}, \nonumber \\
        \zeta_2^*= \mathbb{Z}-\frac{P_2 \mathbb{Z}^2}{\tau}, \nonumber \\
        \hfill \vdots \hfill \nonumber \\
        \zeta_m^*= \mathbb{Z}-\frac{P_m \mathbb{Z}^2}{\tau} \nonumber
    \end{numcases}
    
    By summing up, we obtain the equation of $\mathbb{Z}$:
    \begin{equation}
        \mathbb{Z}=\sum_{n=1}^{M}\zeta_n=M\mathbb{Z}-\frac{\mathbb{Z}^2}{\tau}\sum\limits_{n\in M}P_n
    \end{equation} 
    By solving the Eq.(17), Eq.(12) can be derived. 
    \end{proof}

    Substituting Eq.(12) into Eq.(16), we can derive the optimal strategy for the $LMO_m$:
    \begin{equation}
    \zeta_m^* =  \frac{(M-1)\tau }{\sum_{n\in M}P_{n}}\left ( 1-\frac{(M-1)P_{m}}{\sum_{n\in M}P_{n}} \right )
    \end{equation}
    
    \subsection{Strategy of TP}
    The TP's optimal strategy is setting a budget $\tau$ that maximizes its utility. The optimal $\tau$ must be determined based on the strategies of the LMOs. By selecting the optimal $\tau$, the TP can achieve its maximum utility while the LMOs reach a Nash equilibrium among themselves. Consequently, there exists a Stackelberg equilibrium between the TP and the multiple LMOs. Therefore, we can get the optimal $\tau$ as follows.
    
    \begin{lemma}\label{lemma3}
        There exists a Stackelberg equilibrium between TP and multiple LMOs. 
    \end{lemma}
    \begin{proof}
        We can derive the first-order and second-order derivatives of $\mathcal{U}(t)$:
        \[\frac{\partial \mathcal{U} (\tau )}{\partial \tau }=\lambda g^{'}(\mathbb{Z}^*) \cdot \sum_{n\in M}^{}\frac{\partial \zeta_n^*}{\partial \tau} -1\]
        \[\frac{\partial^2 \mathcal{U}(\tau)}{\partial \tau^2}=\lambda g^{''}(\mathbb{Z}^*)\left [ \sum_{n\in M}^{} \frac{M-1}{\sum_{m\in M}^{}
        p_m}\left ( 1-\frac{(M-1)P_n}{\sum_{m\in M}^{}P_m} \right ) \right ]^2\]
        We have known that $g(\mathbb{Z}^*)$ is a concave function through Eq.\eqref{2}. Therefore $\frac{\partial^2 \mathcal{U}(\tau)}{\partial \tau^2}<0$. Thus if and only $\frac{\partial \mathcal{U} (\tau )}{\partial \tau }=0$, TP's utility reaches the maximum with the optimal $\tau$. Therefore the Stackelberg equilibrium must exists between TP and LMOs.
    \end{proof}

    \begin{theorem}\label{theorem3}
        For TP, its optimal strategy is:
        \begin{equation}\label{12}
            \tau=\lambda \alpha -\frac{\sum_{n\in M}^{}P_n}{\beta (M-1)}
        \end{equation}
    \end{theorem}
    \begin{proof}
    if:
        \[\frac{\partial \mathcal{U} (\tau )}{\partial \tau }=\lambda g^{'}(\mathbb{Z}^*) \cdot \sum_{n\in M}^{}\frac{\partial \zeta_n^*}{\partial \tau} -1 =0\]
    then:
        \[\lambda \cdot \frac{\alpha \beta }{1+\beta \mathbb{Z}^*} \cdot \frac{(M-1) }{\sum\limits_{n\in M}P_n}=1\]
    base on Eq.(12), we can solve the optimal $\tau$.
    \end{proof}
    
    \subsection{The Equilibrium Condition of System}
    In the game between the TP and LMOs, the TP adjusts$\tau$ to maximize its own utility, while each LMO selects its data contribution level $\zeta_m$ to maximize its own utility. Ultimately, this interaction forms a Stackelberg game. The utilities $\mathcal\ {U}(\tau)$ and $\mathcal{U}_m(x_m,\tau,\mathbb{Z})$ reach their respective maxima when the following conditions are satisfied:

    \begin{numcases}{}
        \mathbb{Z}^* =\frac{(M-1)\tau }{\sum_{n\in M}P_n} \tag{12}\\
        \zeta_m^* =  \frac{(M-1)\tau }{\sum_{n\in M}P_{n}}\left ( 1-\frac{(M-1)P_{m}}{\sum_{n\in M}P_{n}} \right )\tag{18}\\ 
        \tau=\lambda \alpha -\frac{\sum_{n\in M}^{}P_n}{\beta (M-1)}\tag{19}
    \end{numcases}
    
    Substituting equation Eq.(19) into equations Eq.(12) and Eq.(18), we can derive the relationship between $\mathbb{Z}^*$ and $x_m$ with respect to $\tau$ and $C_m$:
    
    \begin{numcases}{}
        \zeta_m^* =  \frac{\tau [(\lambda \alpha -\tau )\beta-P_m]}{[(\lambda \alpha -\tau )\beta ]^2} \\ 
        \mathbb{Z}^* = \frac{\tau }{(\lambda \alpha -\tau )\beta }
    \end{numcases}
    
    \subsection{Constraint Analysis}
    After the Nash equilibrium is formed between the LMO and TP, the LMO's budget is established accordingly. To enable the LMO to complete the data collection task, we need to accurately assess the range of the LMO's budget. This is crucial in the final algorithm iteration process. In Definition 2, we provided the expression for the LMO's budget $B_m$. Based on this expression, we thoroughly demonstrate the reasonable value range of $B_m$ in Theorem 3.
    \begin{definition}
    The budget of $LMO_m$ is defined as:
    \[ B_m = \frac{\zeta_m^*}{\mathbb{Z}^*}\tau, (B_m> 0) .\]
    \end{definition}
    Substituting equation Eq.(20) and Eq.(21), we can obtain:
    \begin{equation} \label{15}
        B_m = \left [ 1-\frac{P_m}{(\lambda \alpha -\tau )\beta } \right ]\tau 
    \end{equation}
    We give the range of $P_m$ to determine the range of $B_m$. Lemma 2 can derive the range of $P_m$, and the range of $B_m$ will be given in the Theorem 4.
    
    For any $m\in M$, we have $P_m\in (0,(\lambda \alpha -\tau )\beta )$.
        We consider Eq.(20) as a function of $\zeta_m$ with respect to $P_m$. And the first order derivative is given by:
            $\frac{\mathrm{d} \zeta ^*_m}{\mathrm{d} P_m}=-\frac{\tau}{[(\lambda \alpha -\tau )\beta ]^2}$
        Clearly, the first order derivative is less than zero. And Eq.(20) has a zero point located at $(\lambda\alpha-\tau)\beta$. Given that $\zeta_m$ is greater than zero, thus the range of $P_m$ less than $(\lambda\alpha-\tau)\beta$. Because $P_m$ have to more than zero, the range of $P_m$ can be obtain $P_m\in (0,(\lambda \alpha -\tau )\beta )$.

    \begin{theorem}
        For each $m\in M$, the range of $B_m$ is $(0,\tau)$.
    \end{theorem}
    \begin{proof}
        Similar to LEMMA 2, we consider Eq.(22) as a function of $B_m$ with respect to $P_m$. The first order derivative is :
        \[ \frac{\mathrm{d} B_m}{\mathrm{d} P_m}=-\frac{\tau}{(\lambda \alpha -\tau )\beta }\].
        The domain of function $B_m(P_m)$ is defined as: $P_m\in (0,(\lambda \alpha -\tau )\beta )$. Because of non-negativity, we can obtain the range of $B_m$ as $(0,\tau)$.
    \end{proof}
    
\section{Markov Decision Process and Algorithm Design}
In this section, firstly, we focus on the interaction between the LMO and workers. Secondly, we introduce algorithms that couple all levels of the system. We describe the LMO and workers as a multi-agent Markov decision process, and explain how to solve this problem by using DRL. Finally, we propose algorithms that can couple the global system and explain the rationale behind the algorithm settings and the specific procedures.
    
    \subsection{Markov Decision Process}
    In our proposed three-layer game framework, the interaction between LMOs and the TP leads to a Stackelberg equilibrium described  above, which seeks to maximize its own utility. At equilibrium, $LMO_m$ requires a data contribution of $\zeta_m$. LMOs recruit workers for data collection to complete data demand, while each worker aims to maximize their own reward. For workers, they are unwilling to share information with each other because of privacy concerns. So we model their interactions as a multi-agent Markov process to capture their competitive dynamics. In this subsection, we will provide a detailed discussion on the budget range for LMOs and also give the solution for the multi-agent Markov process.
    
    In each training step $t$, for each $worker_m^i$ are considered to be an agent to make sequential decisions to maximize their utilities in a distributed manner. The states, actions, rewards of the multi-agent Markov decision process are formulated as follows.
        \subsubsection{State Space}
        In each training step $t$, all workers first sets the fatigue level $F^t= \left\{ f_1^t,f_2^t,...,f^t_{W_m}\right\}$ based on the historical data contribution (Eq.(5)). The proportion of the remaining budget $\tilde{B}_m^t$ will represent the normalized value of the resources available in the system. Then, the state space of workers can be defined as:
        \[
        S^t= \left\{ f_1^t,f_2^t,...,f^t_{W_m},\tilde{B}_m^t\right\}.
        \]
        \subsubsection{Action Space}
        In training step $t$, each worker decides whether they will continue to participate in the task. For $worker_i^t$, if they choose to continue, the data contribution is $d_i^t$, and it will follow a normal distribution. The action space of each $worker_i$ is $a^t_i=\{0,1\}$. Thus, the action space of the system is:
        \[
        A^t=\{a_1^t, a_2^t,...,a_{W_m}^t\}.
        \]
        \subsubsection{Reward Function}
        The reward for each workers will be calculated in two stages. In first stage, each agent receives a base reward:
        $$R_{base}^t=
        \begin{cases}
        \theta(1-\frac{\tilde{B}_m}{B_m}), & \text{$a_i^t=1$}\\
        \phi, & \text{$a_i^t=0$}
        \end{cases}$$
        where, $\phi$ is a negative parameter for punishment. In the second stage, for those agents who participate in the task will have a reward based on the data contribution level and reduce the parameter due to fatigue. Thus, the reward function is:
        $$R_{i}^t=
        \begin{cases}
        \frac{d_i^t}{\sum_{i\in W_m}d_m^i}\tilde{B}_m^t-f_i^t + R_{base}^t, & \text{$a_i^t=1$}\\
        \phi, & \text{$a_i^t=0$}
        \end{cases}$$
    \subsection{Multi-Agent Deep Reinforcement Learning Algorithm}
    Considering the characteristics of our environment: First, each agent possesses its own policy network, enabling decentralized decision-making during execution. Second, multiple processes asynchronous collect samples and update policies. Thus, the Asynchronous Proximal Policy Optimization (APPO)\cite{schulman_proximal_2017} algorithm is deployed to achieve the task.

    Specifically, in our scenario setting, each agent's reward depends on the behavior of other agents. For the workers, they all share a common reward pool. This means that all agents simultaneously exist in an environment with limited resources. Whenever an agent receives a reward, the overall budget decreases. Meanwhile, different workers have their own independent fatigue levels and data contribution amounts. This indicates that each agent possesses an independent state, meaning that agents need to make different decisions based on their individual states. Therefore, this is a typical multi-agent environment. Based on these characteristics, we choose the APPO algorithm as our solution, enabling each worker to fully consider both short-term and long-term benefits.

    APPO is an asynchronous variant of Proximal Policy Optimization (PPO) based on the IMPALA architecture. APPO employs a distributed architecture, consisting of Learners, Workers, and Parameter Servers. During training iterations, APPO asynchronously collects samples from the environment and then passes episode references to the Learner to achieve asynchronous model updates. Eq.(24) demonstrates the policy update objective function\cite{schulman_proximal_2017}.

    \begin{equation}
        L^{CLIP}(\theta) = \mathbb{E}_t[\min(r_t(\theta)\hat{A}_t, \text{clip}(r_t(\theta), 1-\epsilon, 1+\epsilon)\hat{A}_t)]
    \end{equation}
    
\subsection{Iterative process for system}
    We trains the workers to learn how to participate in tasks using the APPO algorithm in MDP process. However, there is often a discrepancy between the actual data collected by workers and the theoretically optimal value. To bridge this gap, it is necessary to determine the optimal budget $B_m^*$ required to achieve the theoretical optimal value $\zeta_m^*$. To address this issue, we design an Optimal Budget Search Algorithm (OBSA):
    \begin{algorithm}
\caption{Optimal Budget Search Algorithm (OBSA)}
\label{alg:OBSA}
\KwIn{$W_m,\  \tau,\ \zeta_m$ }
\KwOut{$B_m^*$ }
$B_{\text{min}} = 0$\ $,$ $B_{\text{max}} = \tau$\;
\Repeat{$|\sum_{i\in W_m}^{}d_m^i-\zeta_m| < \varsigma $}{
    $B_m^* = (B_{min}+B_{max})/2$\;
    $\text{Multi-Agent MDP}(B_m^*, W_m)$\;
    \eIf{$\sum_{i\in W_m}^{}d_m^i > \zeta_m$}{
        $B_{\text{max}} = B_m^*$\;
    }{
        $B_{\text{min}} = B_m^*$\;
    }
}
\Return $B_m^*$,${\sum_{i\in W_m}^{}R_i}$, ${\sum_{i\in W_m}^{}d_m^i}$\;
\end{algorithm}
    After $LMO_m$ computes the theoretically optimal strategy, it verifies the strategy's feasibility through OBSA. Specifically, when the $[0,\tau]$ is established as the binary search range (Line 1), the actual data contribution can be calculated by Multi-Agent Markov Decision Process. We approach the theoretical optimal value of the LMO strategy by continuously adjusting LMO's actual budget $B_m^*$ (Line 3-9), and OBSA ultimately determines the actual budget required to complete strategy $\zeta_m$, the total data contribution and the total reward for all workers.
    
    Then, we design the Algorithm 2: Adaptively Searching the Optimal Strategy Algorithm (ASOSA) to address the coupling problem among the three layers for incentive mechanisms of FL.

The purpose of ASOSA is to help the TP, LMOs, and workers maximize their utility. During the ASOSA process, each iteration consists of two phases: TP and LMOs compute the theoretically optimal strategy at the upper layer, while LMOs and workers update parameters based on actual completion at the lower layer. The final result is that the strategies of TP, LMOs, and workers converge to stability.

First, we initialize the number of participating LMOs and the unit data purchase cost C (Lines 1-2). Then, the LMOs broadcast $P_m$ to the TP. The TP aggregates all unit data purchase costs and calculates the theoretically optimal strategy through Eq.(19) (Lines 4-6). If the TP's theoretically optimal strategy calculation yields a non-positive value, the entire game terminates immediately. This indicates that the current LMOs' capabilities are insufficient to complete the training task. The game continues only when the TP's theoretically optimal strategy is greater than zero, and the TP will broadcast its optimal strategy to each LMO (Lines 7-10). Through Eq.(20), the theoretically optimal strategy $\zeta_m$ for each LMO can be calculated. If $\zeta_m$ is less than or equal to zero, the LMO exits the game (Lines 12-14). Through Eq.(22), the LMO can calculate the theoretically optimal budget $B_m$. We use the OBSA algorithm to determine the actual budget $B_m^*$ required for the LMO to complete the current optimal strategy $\zeta_m$, the actual total expenditure $R_{w_m}^{}$, and the actual total data collection amount $d_{w_m}^{}$. Then, we update the unit data purchase cost $P_m$ for $LMO_m$ (Lines 16-18). The entire cycle ends only when all LMOs' strategies stabilize. Finally, each LMO conducts FL training according to its strategy and submits the results to the TP. Notably, the ASOSA algorithm allows the TP and LMOs' utility to approach the theoretical values calculated by the Stackelberg game. For workers, they can maximize their utility through MDP.
\begin{algorithm}
\caption{Adaptively Searching the Optimal Strategy Algorithm (ASOSA)}
\label{alg:ASOSA}
Initialize the number of LMOs $M$\;
Initialize the data purchase cost: $P_{\text{m}}$ for all LMOs\;
\Repeat{No LMOs change their strategy}{
    \For{each ${LMO}_m$}{
        Broadcast $P_{\text{m}}$ to the TP\;
    }
    Determine the optimal $\tau$ based on Eq.(19)\;
    \eIf{$\tau < 0$}{
        Terminate the game\;
    }{
        TP broadcasts $\tau$ to all LMOs\;
    }
    \For{$each\ LMO_m$}{
        Determine the optimal $\zeta_m$ based on Eq.(20)\;
        \eIf{$\zeta_m \leq 0$}{
        $LMO_m$ quit the game\;
        }{
        Determine the optimal $B_m$ based on Eq.(22)\;
        $B_m^*,R_{w_m}^{}, d_{w_m}^{total} = \text{OBSA}(W_m, \tau, \zeta_m)$\;
        $P_m= {R_{w_m}^{}}/{ d_{w_m}^{total}}$\;
        }
    }
}
All LMOs in $M$ train the FL model and submit results\;
\end{algorithm}

\section{Performance Analysis}
\subsection{Experimental Setup}
We have conducted experiments on the Fashion-MNIST\cite{xiao2017fashion} and EMNIST\cite{cohen2017emnist} datasets. These two datasets are used to evaluate the performance of the ASOSA. For the experimental setup, We randomly selected the number of workers for each LMO from a range of 5 to 20. According to previous studies, the parameters are configured as: $ \epsilon =0.1,\gamma =-10.0,\delta =0.5,E_m=4$. For MNIST, we utilized a two-layer neural network with 12 neurons in the hidden layer and 10 neurons in the output layer. For Fashion-MNIST, networks consist of two convolutional layers and two fully connected layers.
\subsection{Parameter Analysis}

\begin{figure}[htbp]    
  \centering            
  \subfloat[]   
  {
      \label{fig:subfig1}\includegraphics[width=0.38\textwidth]{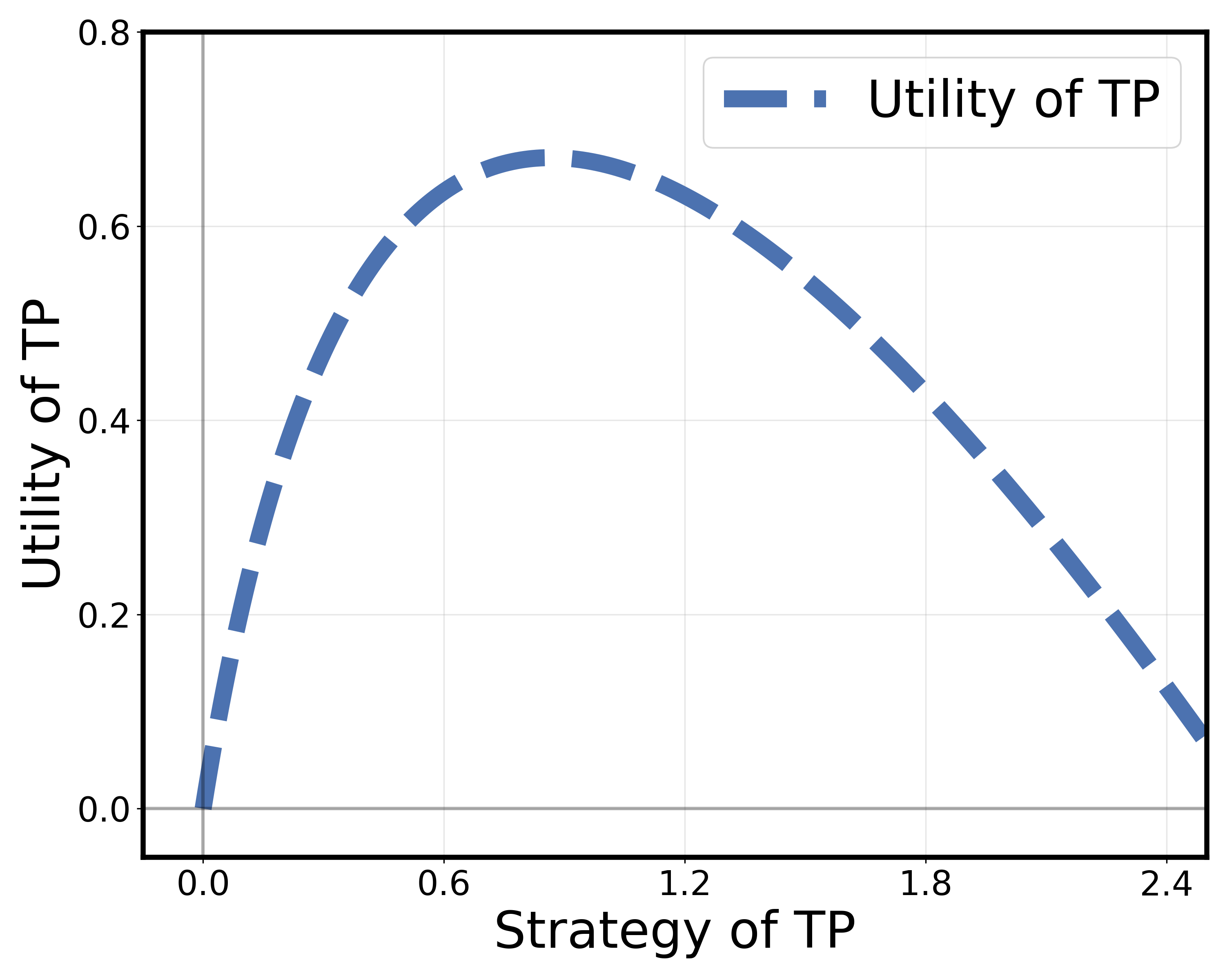}
  }
  \subfloat[]
  {
      \label{fig:subfig2}\includegraphics[width=0.38\textwidth]{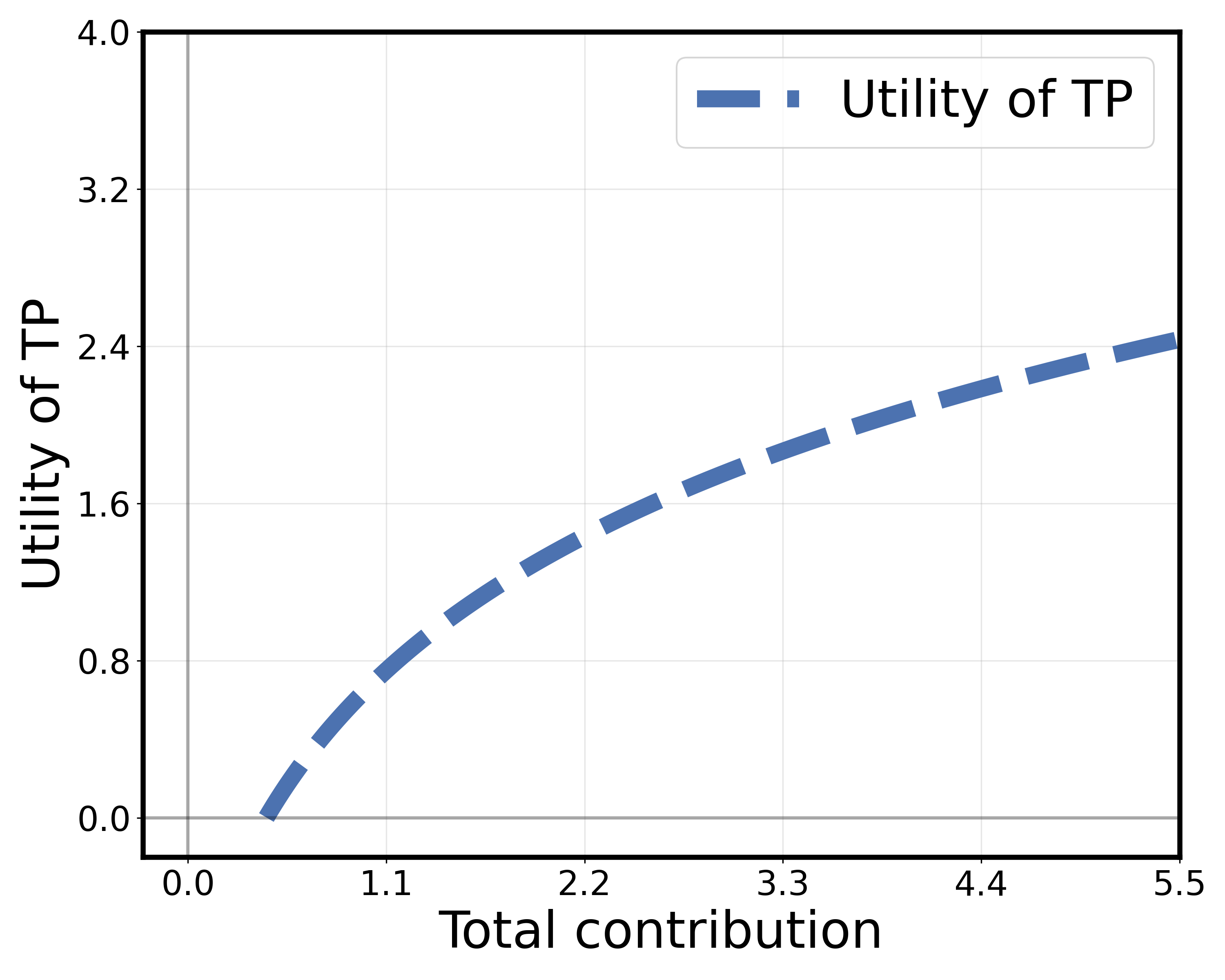}
  }
  \caption{Parameter influencing the utility of TP. (a) The strategy of TP. (b) The total contribution of LMO}    
  \label{fig:subfig_1}            
\end{figure}
\begin{figure}[htbp]    
  \centering            
  \subfloat[]   
{
      \label{fig:subfig1}\includegraphics[width=0.4\textwidth]{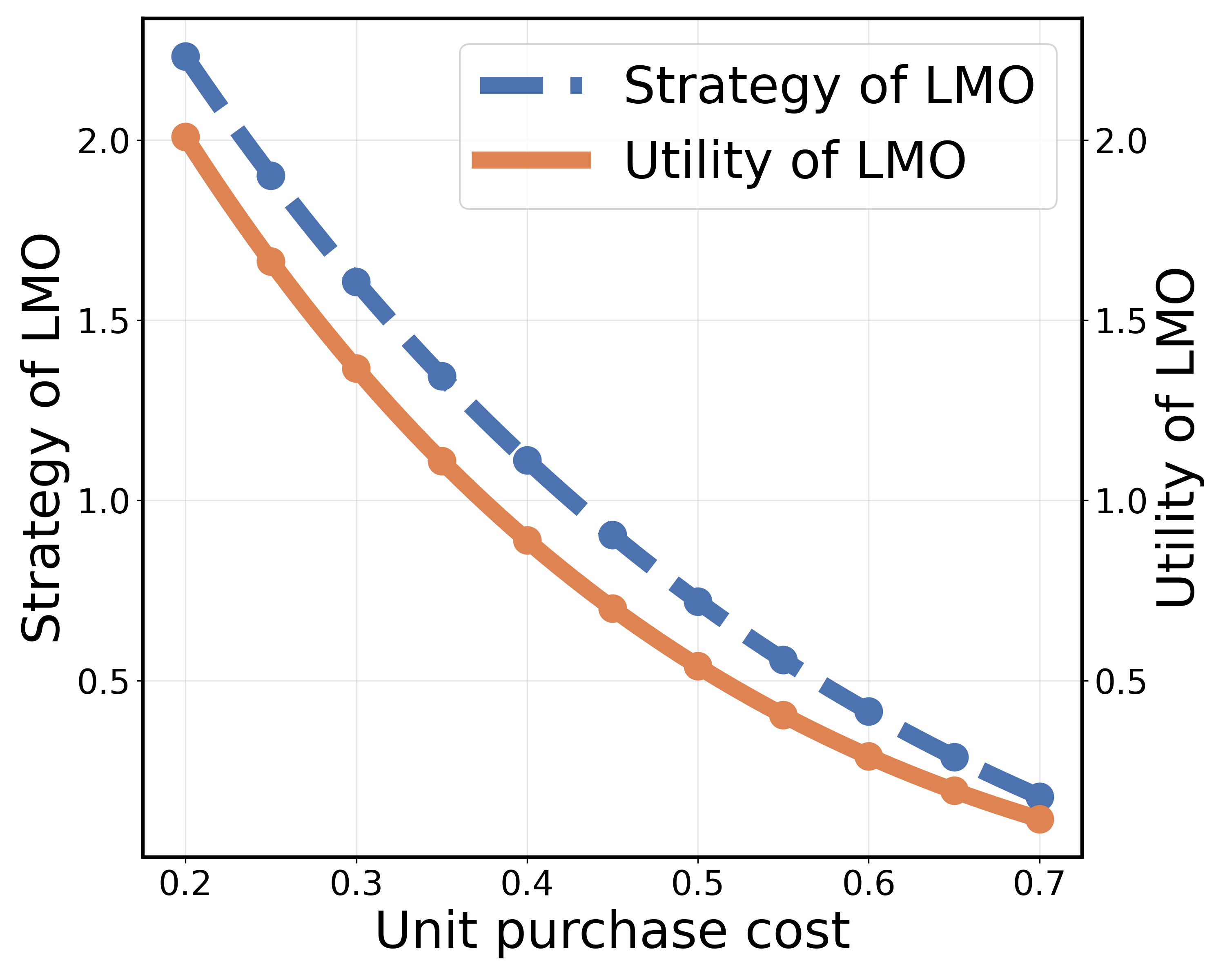}

  }
  \subfloat[]
  {
      \label{fig:subfig2}\includegraphics[width=0.4\textwidth]{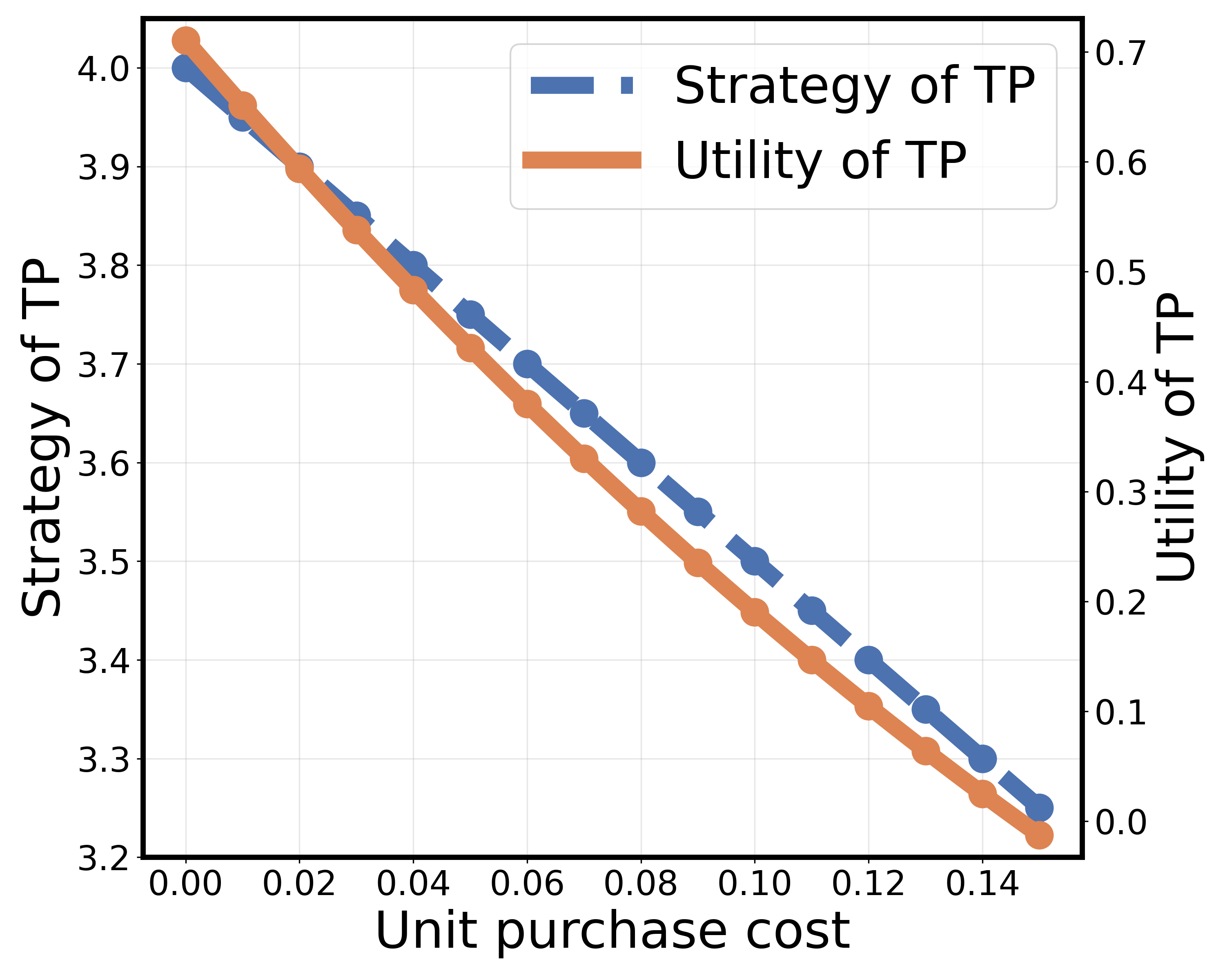}
  }
  \caption{The impact of unit purchase cost. (a) Strategy and utility of LMO. (b) Strategy and utility of TP}    
  \label{fig:subfig_1}            
\end{figure}
For utility of TP, Fig.2 shows the impact of performance when varying the strategy of TP and total contribution of LMOs. In Fig.2(a), when the strategy of TP is decided, all LMOs will determine their optimal strategy, so that we have the utility of the TP under the different strategy of TP. The utility of TP first increases and then decreases, demonstrating clearly that this function possesses a unique maximum point. The maximum point attained by the utility of TP is derived in Theorem 3. This demonstrates that in the Stackelberg game between TP and LMO, TP is able to identify the equilibrium point. In the Fig.2(b), when the TP strategy is determined, the utility of TP is a convex function as the total amount of data increases. This trend corresponds to the accuracy observed in real-world datasets.

Subsequently, we examine the impact of unit data purchase cost on the system. We configure four LMOs, each with an initial unit data acquisition cost of 0.2. We then adjust the unit data purchase cost for one of the LMOs and observe the resulting strategy and utility. From Fig.3(a), we observe that as the unit data purchase cost increases, both the strategy and utility of the LMO decline. This indicates that higher unit data purchase costs intensify the cost pressure on the LMO, thereby reducing the LMO's optimal data collection strategy. Concurrently, this reduction in the optimal data collection strategy directly leads to a decrease in the budget of the LMO, consequently diminishing the utility of the LMO. In Fig.3(b), the experiments demonstrate that an increase in data acquisition costs leads to a decline in both the utility and strategy of the TP. This indicates that the advantage of reduced expenditure for the TP resulting from decreased data contributions is significantly outweighed by the negative impact of diminished model performance benefits.
\subsection{Multi-Agent DRL}
\begin{figure}[htbp]    
  \centering            
  \subfloat[]   
  {
      \label{fig:subfig1}\includegraphics[width=0.4\textwidth]{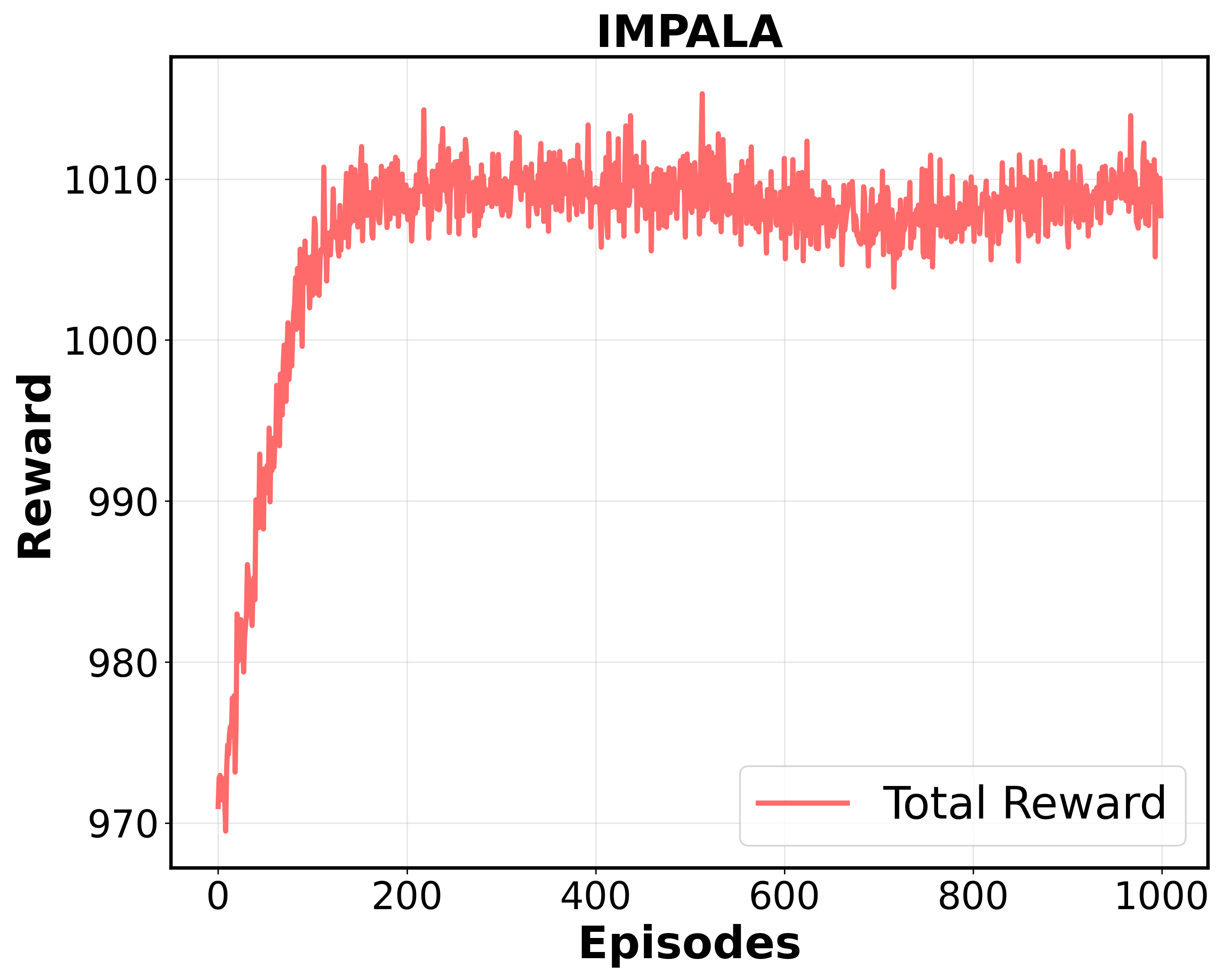}
  }
  \subfloat[]
  {
      \label{fig:subfig2}\includegraphics[width=0.4\textwidth]{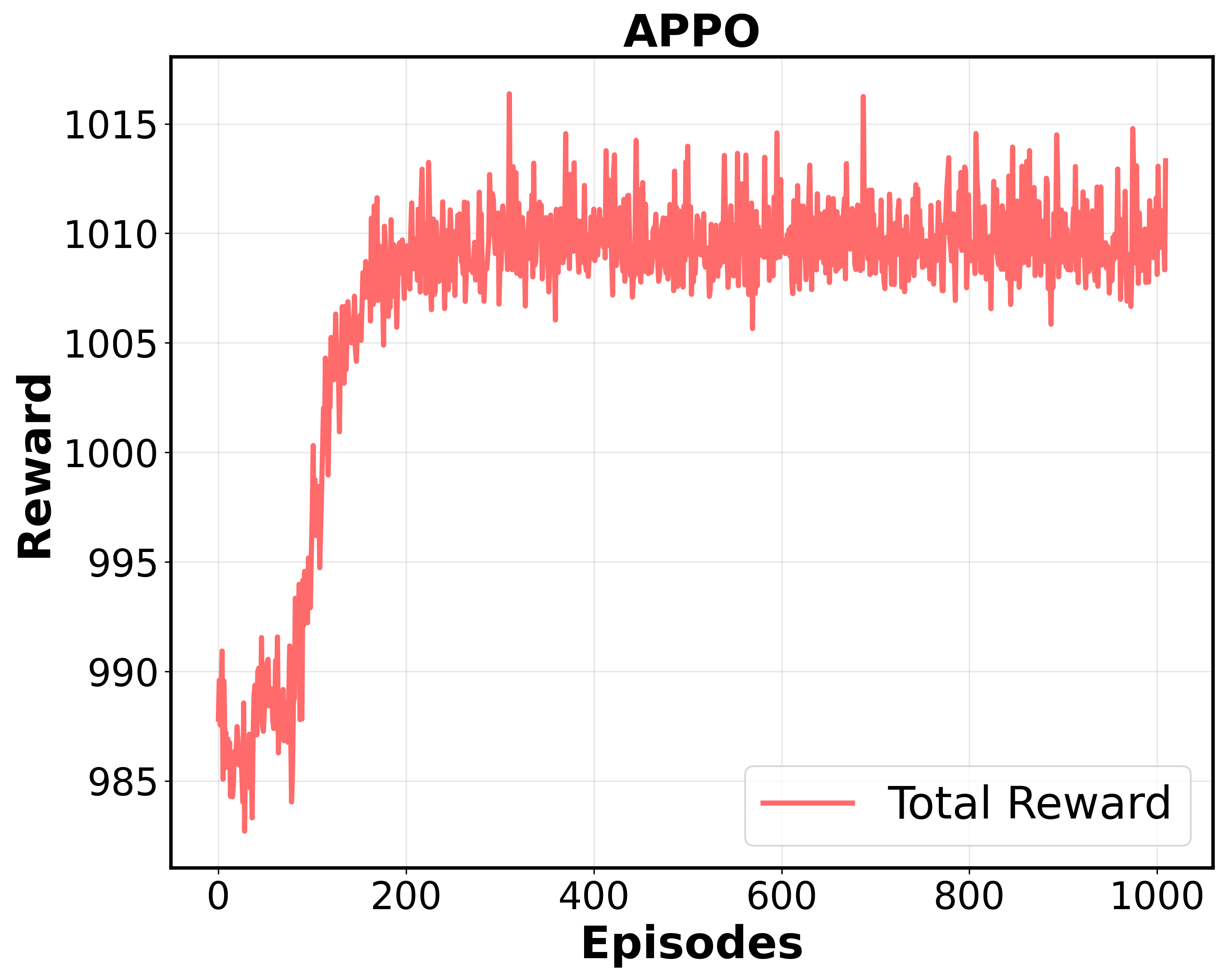}
  }
  \caption{Training rewards (a) IMPALA. (b) APPO.}    
  \label{fig:subfig_1}            
\end{figure}
We train the model following the method described in Section 5, within a customized environment where the number of workers is set to 20. The objective is to find the optimal policy through iterative training over 1,000 episodes. As shown in Fig. 4, the reward curves indicate that the DRL algorithms successfully converged to an optimal and stable policy after sufficient iterations. In addition to APPO, we also experimented with IMPALA as a comparative baseline. The policy learned by APPO slightly outperformed that of IMPALA. Therefore, in the subsequent experiments, we adopt the model trained using APPO.

\subsection{Convergence Analysis}
\begin{figure*}
  \centering
  \begin{subfigure}{0.34\linewidth} 
    \includegraphics[width=0.98\linewidth]{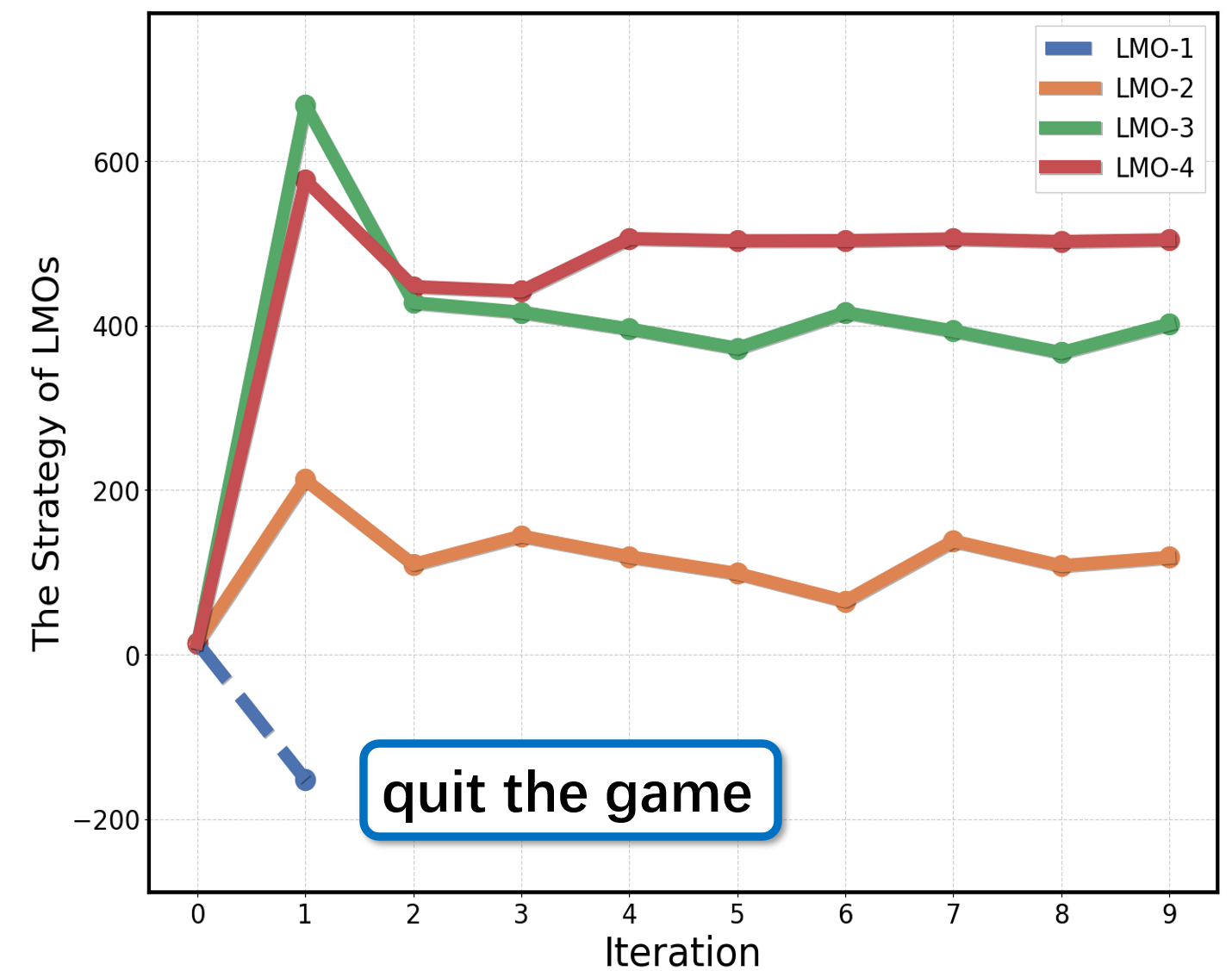} 
    \caption{}
    \label{fig:subfig1}
  \end{subfigure}%
  \begin{subfigure}{0.34\linewidth}
    \includegraphics[width=0.98\linewidth]{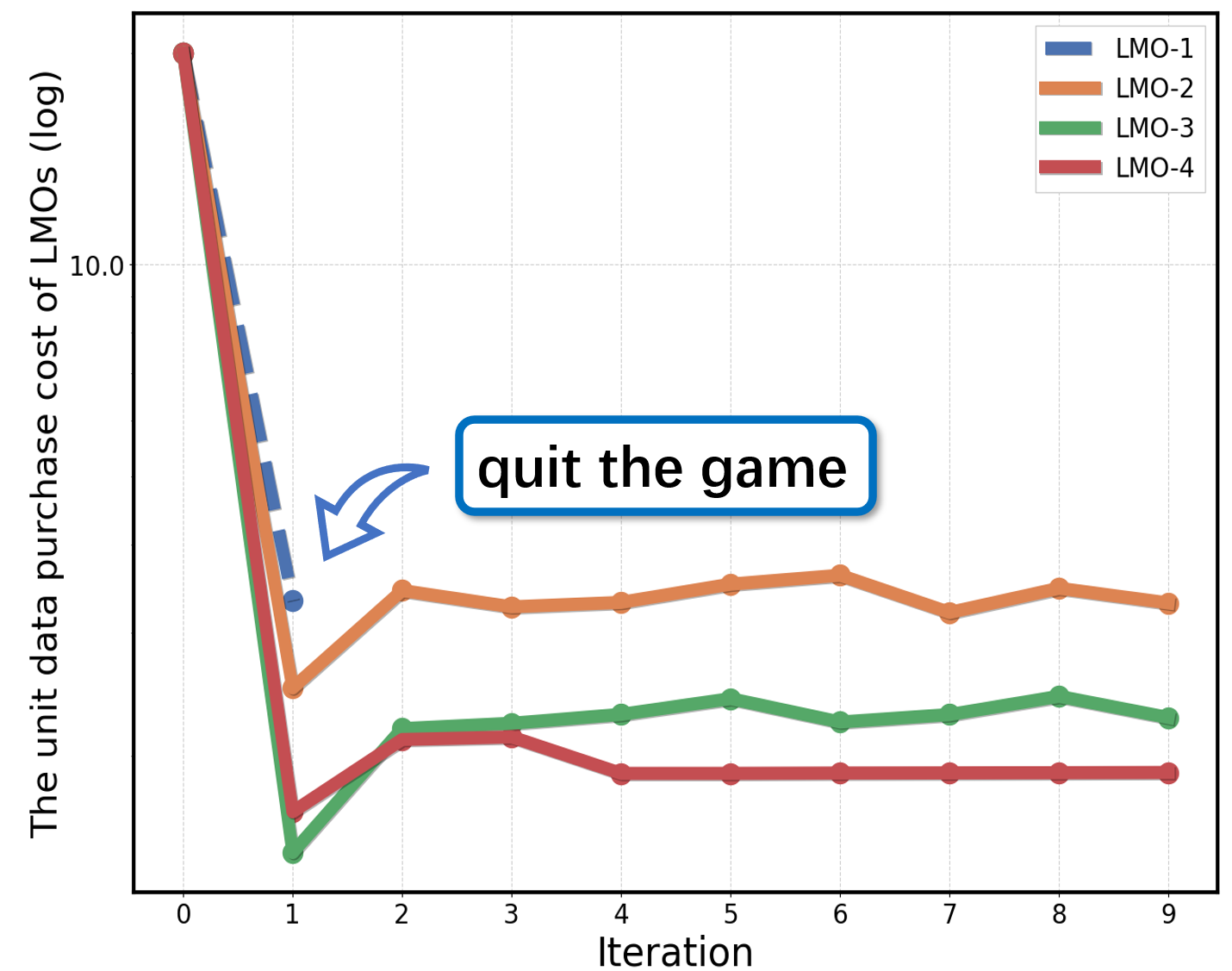} 
    \caption{}
    \label{fig:subfig2}
  \end{subfigure}%
  \begin{subfigure}{0.335\linewidth}
    \includegraphics[width=0.98\linewidth]{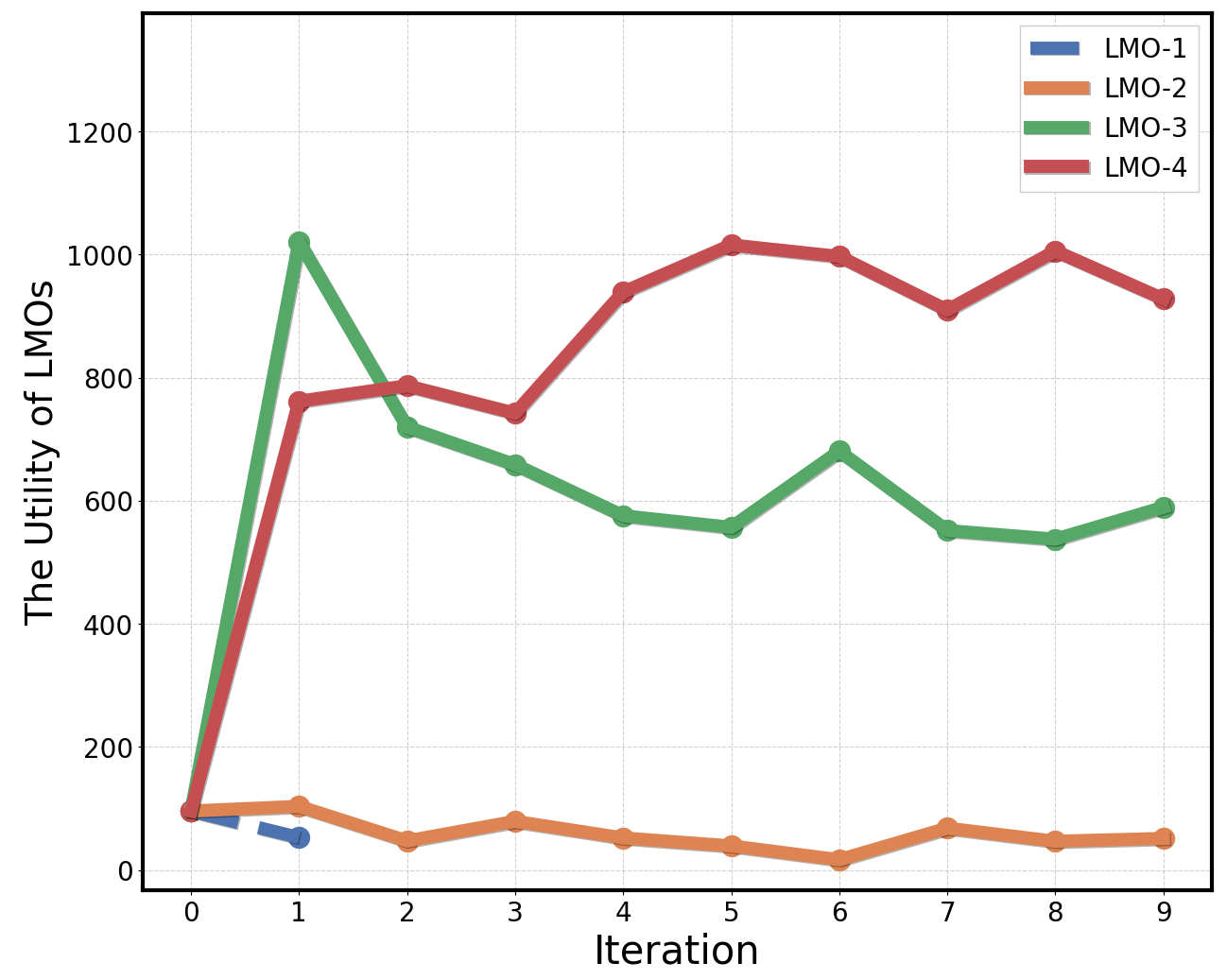} 
    \caption{}
    \label{fig:subfig3}
  \end{subfigure}
  \caption{(a) The strategy of LMOs during iterations. (b) The unit data purchase cost of LMOs during iterations. (c) The utility of LMOs during iterations.}
  \label{fig:wholefig}
\end{figure*}

Fig.5(a) shows the changing of strategies of different LMO, Fig.5(b) shows the unit data purchase cost of LMOs during iterations (we employed the log scale to magnify the details). Fig.5(c) shows the utility of LMOs during iterations. The strategy and utility of TP is shown in the Fig.6. In Fig.4(a), after ten iterations, we can clearly observe that each LMO has converged to a stable strategy.\\
\begin{figure}[htbp]    
  \centering            
{
      \label{fig:subfig1}\includegraphics[width=0.6\textwidth]{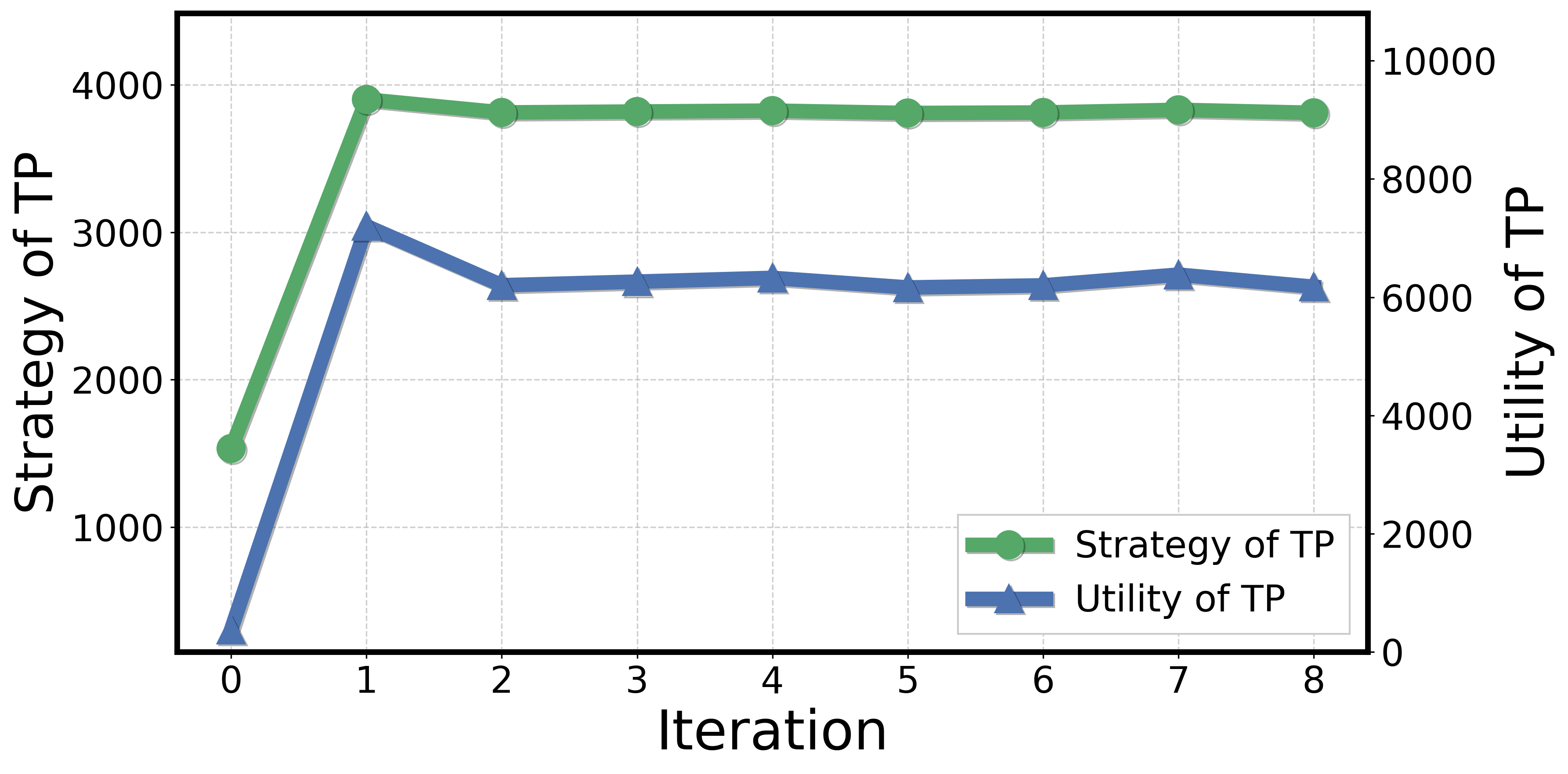}
  }
  \caption{The strategy and utility of TP during iterations}    
  \label{fig:subfig_1}            
\end{figure}

Notably, at the conclusion of the first round, LMO-1 exhibited a negative strategy value. This indicates that LMO-1 was unsuitable for this game and was consequently eliminated. Additionally, LMO-2, LMO-3, and LMO-4 all demonstrated a significant decrease in strategy values during the third round compared to the first round. This occurred because the strategies updated in the first round exceeded the maximum data collection capacity of the workers, resulting in rapid adjustments during the subsequent iteration of ASOSA. For the same reason, the utility of TP briefly increased in the first round before returning to normal levels in Fig.6. These data demonstrate that our algorithm can rapidly adjust strategies according to the characteristics of different LMOs. It ultimately converges quickly to stable strategies.

\subsection{Comparison Experiment}
In this section, we use two benchmark schemes to evaluate the performance in the upper-layer.
\begin{itemize}
	\item Fixed pricing scheme: This scheme implies that the LMO needs to pay the same unit data price for all workers. It is worth noting that the unit data acquisition cost must satisfy the value range determined in the previous analysis. In this section, we set it to 10.
	\item Random pricing scheme: This scheme means to use random number as the price of unit data. For all LMOs, we set this number in the range of $(0,10)$.
\end{itemize}

\begin{figure*}
  \centering
  \begin{subfigure}{0.34\linewidth} 
    \includegraphics[width=0.98\linewidth]{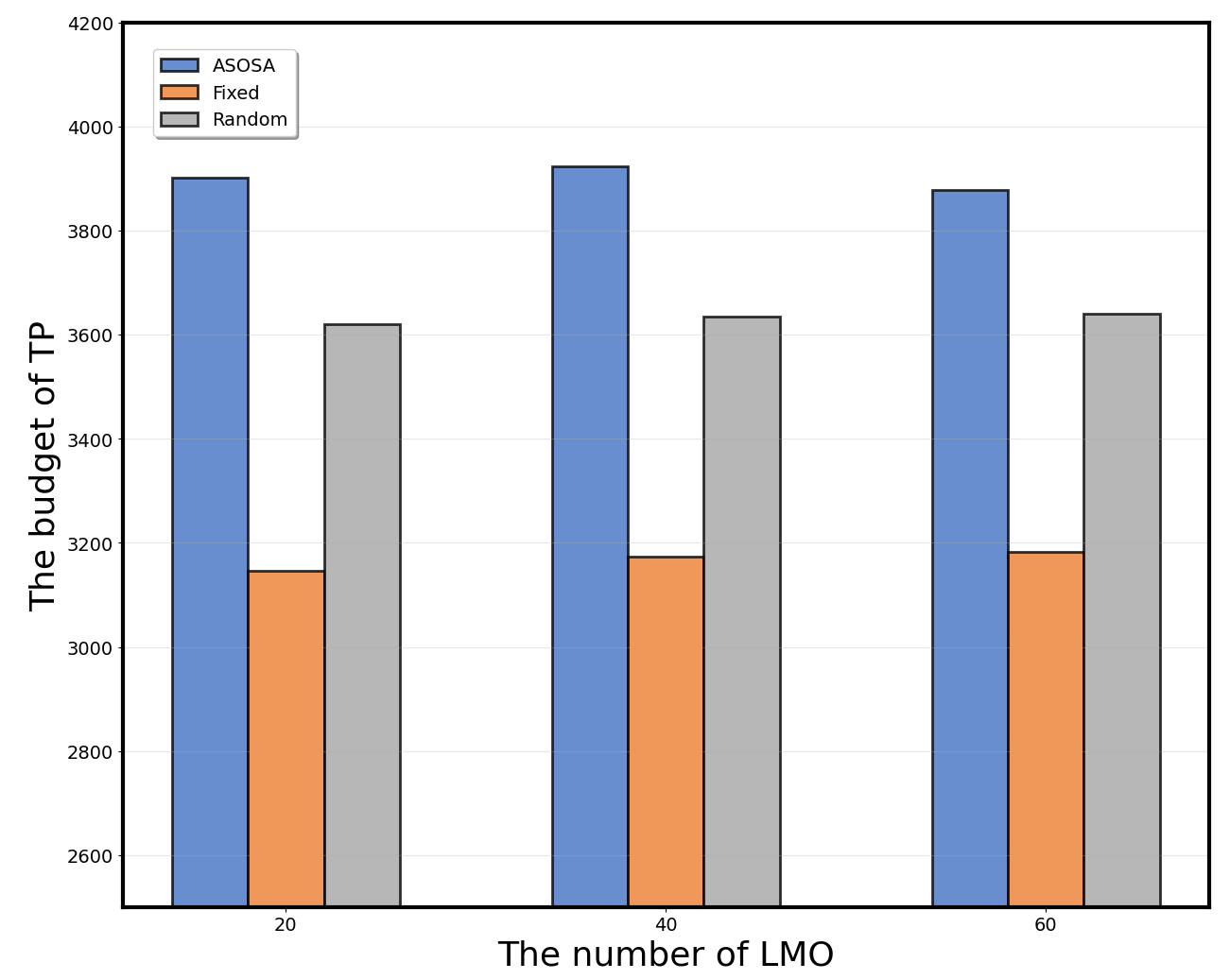} 
    \caption{}
    \label{fig:subfig1}
  \end{subfigure}%
  \begin{subfigure}{0.34\linewidth}
    \includegraphics[width=0.98\linewidth]{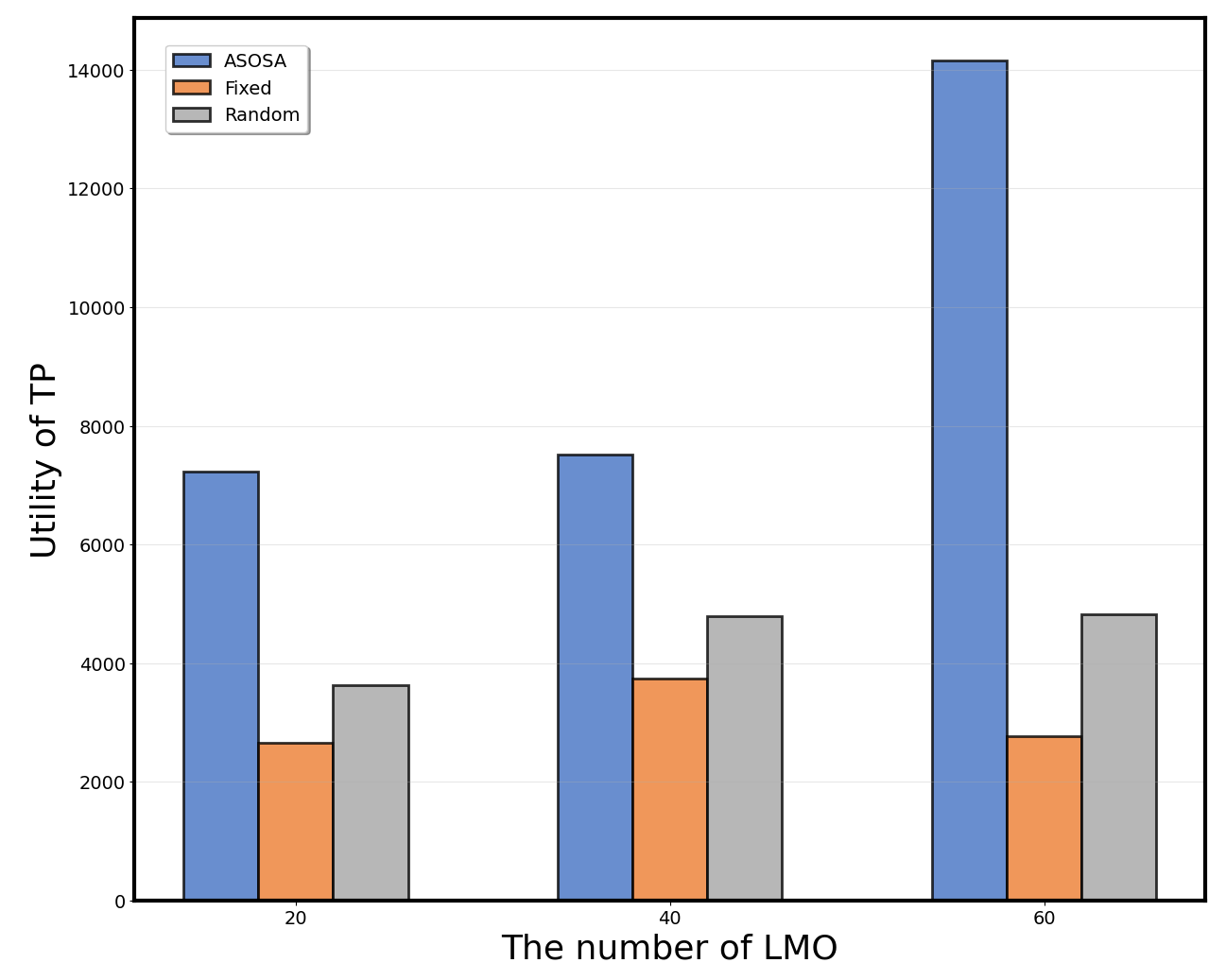} 
    \caption{}
    \label{fig:subfig2}
  \end{subfigure}%
  \begin{subfigure}{0.34\linewidth}
    \includegraphics[width=0.98\linewidth]{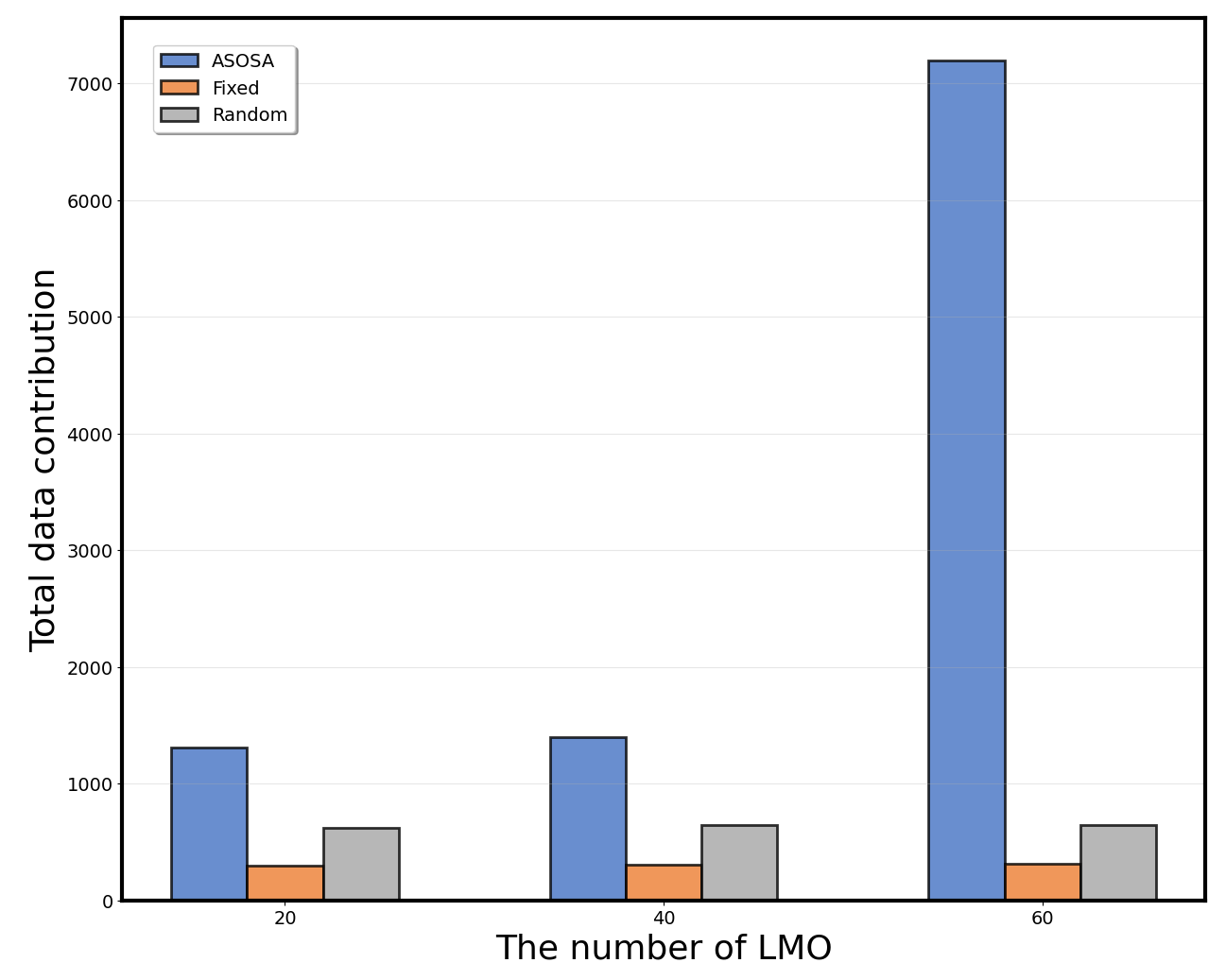} 
    \caption{}
    \label{fig:subfig3}
  \end{subfigure}
  \caption{(a) The budget of TP. (b) The utility of TP. (c) Total data contribution of LMOs.}
  \label{fig:wholefig}
\end{figure*}

Fig.7(a) shows the total budget of the TP when the strategy stabilizes under different algorithms. Fig.7(b) demonstrates how different algorithms affect the TP's utility. Fig.7(c) shows the impact of total data contribution $\mathbb{Z}$ by using different algorithms. We can clearly observe that as the number of LMOs increases, the TP's budget does not increase significantly but both the TP's utility and the total amount of data collected by the LMOs are significantly improved. This indicates that ASOSA can effectively find strategies that achieve optimal utility across various scenarios without significantly increasing the TP's budget. Meanwhile, the utility achieved through ASOSA substantially exceeds that of both random pricing strategies and fixed pricing strategies.

\section{Conclusion}
In our research, we resolve the coupling problem between theoretical optimal solutions in the upper layer and actual values in the lower layer. Through iterations of unit data acquisition cost, we develop the ASOSA algorithm to adaptively adjust strategies based on the characteristics of different LMOs. Ultimately, ASOSA demonstrates excellent stability in experiments. It is able to iterate towards optimal strategies under various conditions, significantly outperforming other baseline pricing strategies. 

\bibliography{Writing.bib}

\begin{thebibliography}{10}
\providecommand{\url}[1]{\texttt{#1}}
\providecommand{\urlprefix}{URL }
\providecommand{\doi}[1]{https://doi.org/#1}

\bibitem{gaff2014privacy}
Gaff, B.M., Sussman, H.E., Geetter, J.: Privacy and big data. Computer  \textbf{47}(6), ~7--9 (2014)

\bibitem{wu2014service}
Wu, Z.: Service Computing: Concept, Method and Technology. Academic Press (2014)

\bibitem{liu2024blockchain}
Liu, J., Sun, C.a., Wu, T., Aiello, M.: Blockchain-based privacy-preserving data service provisioning for internet of things. In: 2024 IEEE International Conference on Web Services (ICWS). pp. 524--534. IEEE (2024)

\bibitem{li2020federated}
Li, T., Sahu, A.K., Zaheer, M., Sanjabi, M., Talwalkar, A., Smith, V.: Federated optimization in heterogeneous networks. Proceedings of Machine learning and systems  \textbf{2},  429--450 (2020)

\bibitem{yang2019federated}
Yang, Q., Liu, Y., Chen, T., Tong, Y.: Federated machine learning: Concept and applications. ACM Transactions on Intelligent Systems and Technology (TIST)  \textbf{10}(2),  1--19 (2019)

\bibitem{mcmahan2017communication}
McMahan, B., Moore, E., Ramage, D., Hampson, S., y~Arcas, B.A.: Communication-efficient learning of deep networks from decentralized data. In: Artificial intelligence and statistics. pp. 1273--1282. PMLR (2017)

\bibitem{karimireddy2020scaffold}
Karimireddy, S.P., Kale, S., Mohri, M., Reddi, S., Stich, S., Suresh, A.T.: Scaffold: Stochastic controlled averaging for federated learning. In: International conference on machine learning. pp. 5132--5143. PMLR (2020)

\bibitem{hard2018federated}
Hard, A., Rao, K., Mathews, R., Ramaswamy, S., Beaufays, F., Augenstein, S., Eichner, H., Kiddon, C., Ramage, D.: Federated learning for mobile keyboard prediction. arXiv preprint arXiv:1811.03604  (2018)

\bibitem{zhang2021dynamic}
Zhang, W., Zhou, T., Lu, Q., Wang, X., Zhu, C., Sun, H., Wang, Z., Lo, S.K., Wang, F.Y.: Dynamic-fusion-based federated learning for covid-19 detection. IEEE Internet of Things Journal  \textbf{8}(21),  15884--15891 (2021)

\bibitem{li2021privacy}
Li, Y., Tao, X., Zhang, X., Liu, J., Xu, J.: Privacy-preserved federated learning for autonomous driving. IEEE Transactions on Intelligent Transportation Systems  \textbf{23}(7),  8423--8434 (2021)

\bibitem{liu2020systematic}
Liu, Y., Zhang, L., Ge, N., Li, G.: A systematic literature review on federated learning: From a model quality perspective. arXiv preprint arXiv:2012.01973  (2020)

\bibitem{rb1}
Lim, W.Y.B., Luong, N.C., Hoang, D.T., Jiao, Y., Liang, Y.C., Yang, Q., Niyato, D., Miao, C.: Federated learning in mobile edge networks: A comprehensive survey. IEEE communications surveys \& tutorials  \textbf{22}(3),  2031--2063 (2020)

\bibitem{rb2}
Rauniyar, A., Hagos, D.H., Jha, D., H{\aa}keg{\aa}rd, J.E., Bagci, U., Rawat, D.B., Vlassov, V.: Federated learning for medical applications: A taxonomy, current trends, challenges, and future research directions. IEEE Internet of Things Journal  \textbf{11}(5),  7374--7398 (2023)

\bibitem{rb3}
Zhang, W., Yu, F., Wang, X., Zeng, X., Zhao, H., Tian, Y., Wang, F.Y., Li, L., Li, Z.: R$^{2}$fed: Resilient reinforcement federated learning for industrial applications. IEEE Transactions on Industrial Informatics  \textbf{19}(8),  8829--8840 (2023). \doi{10.1109/TII.2022.3222369}

\bibitem{rb4}
Dai, C., Wei, S., Dai, S., Garg, S., Kaddoum, G., Shamim~Hossain, M.: Federated self-supervised learning based on prototypes clustering contrastive learning for internet of vehicles applications. IEEE Internet of Things Journal  \textbf{12}(5),  4692--4700 (2025). \doi{10.1109/JIOT.2024.3453336}

\bibitem{rb5}
Wang, C., Zhou, Z., Zhang, X., Chen, X.: Bridging the data gap in federated preference learning with aigc. In: 2024 IEEE 44th International Conference on Distributed Computing Systems (ICDCS). pp. 105--116. IEEE (2024)

\bibitem{rb6}
Zhang, M., Wang, S.: Matrix sketching for secure collaborative machine learning. In: International Conference on Machine Learning. pp. 12589--12599. PMLR (2021)

\bibitem{rb10}
Zhang, W., Wang, Q., Zhao, H., Xia, W., Zhu, H.: Incentivizing quality contributions in federated learning: A stackelberg game approach. In: 2024 IEEE 99th Vehicular Technology Conference (VTC2024-Spring). pp.~1--5. IEEE (2024)

\bibitem{rb11}
Cho, M.C., Yen, L.H., Wang, J.X.: Seeking stability for multi-leader stackelberg game as an incentive mechanism for multi-requester federated learning. IEEE Access  (2025)

\bibitem{rb15}
Ding, N., Gao, L., Huang, J.: Incentive mechanism design for federated learning with dynamic network pricing. IEEE Transactions on Mobile Computing  (2025)

\bibitem{rb12}
Li, G., Cai, J., Lu, J., Chen, H.: Incentive mechanism design for cross-device federated learning: A reinforcement auction approach. IEEE Transactions on Mobile Computing  (2024)

\bibitem{rb13}
Wu, L., Guo, S., Hong, Z., Liu, Y., Xu, W., Zhan, Y.: Long-term adaptive vcg auction mechanism for sustainable federated learning with periodical client shifting. IEEE Transactions on Mobile Computing  \textbf{23}(5),  6060--6073 (2023)

\bibitem{li2024dynamic}
Li, J., Ji, S., Jin, H., Dong, H., Ge, Z., Zhang, P.: Dynamic adaptive user allocation in mobile edge computing. In: 2024 IEEE International Conference on Software Services Engineering (SSE). pp. 179--187. IEEE (2024)

\bibitem{rb14}
Liu, J., Li, X., Xu, Y., Lyu, C., Wang, Y., Liu, X.: Hedonic coalition formation game and contract-based federated learning in uav-assisted internet of things. IEEE Internet of Things Journal  (2025)

\bibitem{zhao2023multi}
Zhao, N., Pei, Y., Liang, Y.C., Niyato, D.: Multi-agent deep reinforcement learning based incentive mechanism for multi-task federated edge learning. IEEE Transactions on Vehicular Technology  \textbf{72}(10),  13530--13535 (2023)

\bibitem{huang_hierarchical_2024}
Huang, J., Ma, B., Wu, Y., Chen, Y., Shen, X.: A {Hierarchical} {Incentive} {Mechanism} for {Federated} {Learning}. IEEE Transactions on Mobile Computing  \textbf{23}(12),  12731--12747 (Dec 2024). \doi{10.1109/TMC.2024.3423399}, \url{https://ieeexplore.ieee.org/document/10586269/}

\bibitem{schulman_proximal_2017}
Schulman, J., Wolski, F., Dhariwal, P., Radford, A., Klimov, O.: Proximal {Policy} {Optimization} {Algorithms} (Aug 2017). \doi{10.48550/arXiv.1707.06347}, \url{http://arxiv.org/abs/1707.06347}, arXiv:1707.06347 [cs]

\bibitem{xiao2017fashion}
Xiao, H., Rasul, K., Vollgraf, R.: Fashion-mnist: a novel image dataset for benchmarking machine learning algorithms. arXiv preprint arXiv:1708.07747  (2017)

\bibitem{cohen2017emnist}
Cohen, G., Afshar, S., Tapson, J., Van~Schaik, A.: Emnist: Extending mnist to handwritten letters. In: 2017 international joint conference on neural networks (IJCNN). pp. 2921--2926. IEEE (2017)

\end{thebibliography}

\end{document}